\documentclass[10pt]{article}

\usepackage[paperwidth=7.25in, paperheight=10in, textwidth=5.25in, textheight=8.2in]{geometry}

\usepackage{subfig} 
\usepackage{float}
\usepackage{amsmath}
\usepackage{amsthm}
\usepackage{amssymb}
\usepackage[sort,numbers]{natbib}
\usepackage{graphicx}
\usepackage{booktabs}
\usepackage{mathtools}
\usepackage{xr,tikz}
\usepackage{enumitem}
\usepackage{nicefrac}
\usepackage{forloop}
\usepackage[hidelinks]{hyperref}
\usepackage{url}
\usepackage{bbm} 
\usepackage{authblk}
\usepackage[capitalise]{cleveref}
\usepackage{environ}
\usepackage{bm,xspace}
\usepackage{microtype}
\renewcommand{\vec}[1]{\mathbf{#1}}
\usepackage{forloop}
\newcommand{\defcal}[1]{\expandafter\newcommand\csname c#1\endcsname{{\mathcal{#1}}}}
\newcommand{\defbb}[1]{\expandafter\newcommand\csname b#1\endcsname{{\mathbb{#1}}}}
\newcounter{calBbCounter}
\forLoop{1}{26}{calBbCounter}{
	\edef\letter{\Alph{calBbCounter}}
	\expandafter\defcal\letter
	\expandafter\defbb\letter
}

\newtheorem{theorem}{Theorem}

\newtheorem{definition}{Definition}

\newcommand{\Acc}{\text{Acc.}}
\newcommand{\Dis}{\text{G}_{\text{Discr.}}}
\newcommand{\LDis}{\text{L}_{\text{Discr.}}}
\newcommand{\Knn}{\text{kNN-Pred}}

\newcommand{\optimal}{h^*}
\newcommand{\optimalfair}{h^{*}_{\textup{fair}}}
\usepackage[capitalise]{cleveref}
\usepackage{amssymb}
\usepackage{algorithm, algpseudocode, xspace}
\usepackage{tabularx,ragged2e,booktabs}
\theoremstyle{remark}

\makeatother

\title{Eliminating Latent Discrimination: \\Train Then Mask}
\usepackage{times}

\author[1]{Soheil Ghili\footnote{The first two authors have equally contributed to this paper.}}
\newcommand\CoAuthorMark{\footnotemark[\arabic{footnote}]} 
\author[2]{Ehsan Kazemi\protect\CoAuthorMark\footnote{Corresponding author.}}
\author[2]{Amin Karbasi}
\affil[1]{School of Management\\Yale University}
\affil[2]{Yale Institute for Network Science\\Yale University}
\affil[ ]{\normalsize \texttt{\{soheil.ghili, ehsan.kazemi, amin.karbasi\}@yale.edu}}

\date{}

\begin{document}
\maketitle

\begin{abstract}
	How can we control for latent discrimination in predictive models? How can we provably remove it? Such questions are at the heart of algorithmic fairness and its impacts on society. In this paper, we define a new operational fairness criteria, inspired by the well-understood notion of omitted variable-bias in statistics and econometrics. Our notion of fairness effectively controls for sensitive features and provides diagnostics for deviations from fair decision making. We then establish analytical and algorithmic results about the existence of a fair classifier in the context of supervised learning. Our results readily imply a simple, but rather counter-intuitive, strategy for eliminating latent discrimination.  In order to prevent other features proxying for sensitive features, we need to include sensitive features in the training phase, but exclude them in the test/evaluation phase while controlling for their effects. We evaluate the performance of our algorithm on several real-world datasets and show how fairness for these datasets can be improved with a very small loss in accuracy. 

\end{abstract}

\section{Introduction}
Nowadays, many sensitive decision-making tasks rely on automated statistical and machine learning algorithms.
Examples include  targeted advertising, credit scores and loans, college admissions, prediction of domestic violence, and even investment strategies for venture capital groups. 
There has been a growing concern about errors, unfairness, and transparency of such mechanisms from governments, civil organizations and research societies \citep{whiteHouse2016big, barocas2016big, compas2018data}.  
That is, whether or not we can  prevent discrimination against protected groups and attributes  (e.g., race, gender, etc). 
Clearly, training a machine learning algorithm with the standard aim of loss function minimization (i.e., high accuracy, low prediction error, etc) may result in predictive behaviors that are unfair towards certain groups or individuals \citep{hardt2016equality, liu2018delayed, zhang2018fairness}.

In many real-world applications, we are not allowed to use some sensitive features. For example, EU anti-discrimination law prohibits the use of protected attributes (directly or indirectly) for several decision-making tasks \citep{ellis2012eu}. 
A naive approach towards fairness is to discard sensitive attributes from training data.  However, if other (seemingly) non-sensitive variables are correlated with the protected ones, the learning algorithm may use them to \emph{proxy for protected features} in order to achieve a lower loss.\footnote{In \cref{sec:experiments} we observe that in many datasets: (i) the admission rate is hugely against the protected groups, and (ii) there are several features that are tightly correlated with the sensitive attribute.}
We call this a \emph{latent} form of discrimination. 
Mitigating this kind of latent discrimination has received considerable attention in the machine learning community and interesting heuristic algorithms have been proposed (e.g., \citet{zemel2013learning} and \citet{kamiran2009classifying}). 
Since this type of discrimination is latent, most previous works fail to provide an operational description of this notion and usually resort to descriptive statements. 

The first contribution of  this  paper is to propose a new operational  definition of fairness, called \emph{EL-fairness} (it stands for explicit and latent fairness), that controls for sensitive features. 
This definition rules out explicit discrimination in the conventional way by treating individuals with similar non-sensitive features similarly. 
It also provides a detection mechanism for observing latent discrimination of a classifier by comparing simple statistics within protected-feature groups to the ones provided by the optimum unconstrained classifier (trained on the full data set).

Proxying or \emph{omitted variable bias} (OVB) occurs when a feature which is correlated with some other attributes is left out. 
In many models, for example, linear regression, it is well known that provided enough data, keeping the sensitive feature controls for OVB which enables us to separate its effect from other correlated attributes \citep{seber2012linear}.
Building on our notion of EL-fairness and existing methods to remove/reduce  the proxying effect or OVB (e.g., \citet{vzliobaite2016using}), we develop a procedure for obtaining fair classifiers. 
In particular, we show that in order to  eliminate latent discrimination one needs to consider the sensitive features in the training phase (in order to obtain reliable statistics to control for such features) and then mask them in evaluation/test phase. 
This way, we can ensure that correlated variables do not proxy the sensitive features and, more importantly, decisions are not made based on protected attributes. 
Furthermore, such a train-then-mask approach achieves EL-fairness with almost no additional computational cost as the training phase is intact. 

More specifically, in this paper we make two algorithmic contributions: (i) keep the sensitive feature during the training phase to control for OVB, and (ii) find an EL-fair classifier with the maximum accuracy by choosing the parameters of our algorithm properly.
We use this idea to control for OVB in a general class of separable functions.\footnote{Note that linear and logistic regressions are two simple members of this general class which we define later.}

As a final note, we should point out that our notion of fairness is robust against \emph{double discrimination}. 
This is a peculiar situation (that happens surprisingly often) where a minority group outperforms the rest of the population despite discrimination. 
We show that our proposed procedure still removes the bias against the protected group in such scenarios, while group-fairness based notions do not.

The rest of this paper is organized as follows. 
In \cref{sec:related}, we review the related literature. 
In \cref{sec:setup}, we define our notion of EL-fairness.
In \cref{sec:characterization}, we characterize the existence and properties of the optimal fair classifier and explain the train-then-mask algorithm.  We prove that under EL-fairness, the optimal fair classifier is equivalent to first training the model with the sensitive feature included, and then omitting the sensitive feature afterward. This way the sensitive feature is \emph{controlled for} but absent from the final outcome.
In \cref{sec:discussion}, we discuss the relation of  EL-fairness with double-unfairness and separability.
In \cref{sec:experiments}, we perform an exhaustive set of empirical studies to establish that our proposed approach reliably reduces latent discrimination with little loss in accuracy.
 
\section{Related Work}\label{sec:related}
This paper brings together pieces of literature from econometrics and machine learning. 
It is well known in both fields that if a variable is omitted from an analysis (on purpose or because it is unobservable), it might distort the results of the analysis in case it is correlated with variables that are not omitted \citep{greene2003econometric,DIKL2018,kamiran2009classifying,dwork2012fairness, hardt2016equality}. 
In the econometrics literature, this concern arises mainly in the context of \emph{causal inference} where the main objective is to estimate a \emph{treatment effect}.
There, if there is a factor that (i) impacts the outcome, (ii) is omitted from the analysis, and (iii) is correlated with the treatment, it can bias the estimated treatment effect, since (part of) the effect of the unobserved variable may be picked up by the estimation process as the effect of the treatment. 
This is generally called the \emph{Omitted Variable Bias} \citep{greene2003econometric}. 
The typical solution is to incorporate such variables as controls in the statistical model. This is an integral part of empirical and experimental research in multiple fields such as econometrics, marketing, and medicine \citep{sudhir2001structural, clarke2005phantom, scheffler2007role, hendel2013intertemporal, ghili2016network}. 

However, the same strategy (i.e., controls) has not been explicitly used in the context of fairness in machine learning.
This, partly stems from the fact that the objective functions are more complicated in the real world applications that  machine learning algorithms aim to solve. 
On the one hand, it is not desirable if omitting a sensitive feature leads to latent discrimination, since correlated (and seemingly) non-sensitive features now  can act as proxies for the sensitive one \citep{hardt2016equality, pedreshi2008discrimination}. 
Unlike the causal inference literature, however, this problem may not be resolved by incorporating the sensitive feature in the analysis. 
This would eliminate latent discrimination but would come at the larger expense of explicit discrimination, i.e., the model might treat similar individuals of two different groups differently. 
The approaches suggested in the fairness literature to deal with this problem have either been mainly based on relabeling the data \citep{kamiran2009classifying} or based on mapping the data to a set of prototypes \citep{dwork2012fairness}. 
These approaches attempt to eliminate latent discrimination by directly or indirectly entering a notion of group fairness into the objective function of the optimization problem. That is, for instance, they try to achieve high admission accuracy but restrict the ratio of the number of admitted individuals form the unprotected group over admitted ones form the protected group.

More recently, \citet{kilbertus2017avoiding} and \citet{nabi2018fair} framed the problem of fairness based on sensitive features in the language of causal reasoning in order to  resolve the effect of proxy variables. 
Their focus is on the theoretical analysis of cases in which the full causal relationships among all (sensitive and nonsensitive) features are precisely known. 
In addition, \citet{zhang2018fairness} proposed a causal explanation formula to quantitatively evaluate fairness.
We should point out that while the structure of causality could be learned by data generating models (e.g., in some special cases under certain linearity assumptions), our approach does not require such information.

Furthermore, several studies (such as the works of \citet{hu2018short} and \citet{liu2018delayed}) consider the long-term effect of classification on different groups in the population. For another instance, \citet{jabbari2017fairness} investigated the long-term price of fairness in reinforcement learning.
Similarly, \citet{gillen2018online} considered the fairness problem in online learning scenarios where the main objective is to minimize a game theoretic notion of regret. 
Also, fairness is studied in many other machine learning settings, including ranking \citep{celis2018ranking}, personalization and recommendation \citep{celis2017fair, kamishima2018recommendation, burke2018balanced}, data summarization \citep{celis2018fair}, targeted advertisement \citep{speicher2018potential}, fair PCA \citep{samadi2018price}, empirical risk minimization \citep{donini2018empirical,hashimoto2018fairness}, privacy preserving \citep{ekstrand2018privacy} and a welfare-based measure of fairness \citep{heidari2018fairness}.
Finally,  due to the massive size of today's datasets, practical algorithms with fairness criteria should be able to scale. To this end, \citet{hlaca2018beyond} and \citet{kazemi2018deletion} have developed several scalable  methods with the aim of preserving fairness in their predictions.

 Our approach has important implications for other notions of fairness such as group fairness and individual fairness. 
It is well-known that there are inherent trade-offs among different notions of fairness and therefore satisfying multiple fairness criteria simultaneously is not possible \citep{KMR2017,pleiss2017on}.
For example, all methods that aim at solving the issue of proxying (including ours)  do not satisfy the calibration property. 

\section{Setup and Problem Formulation}\label{sec:setup}
Let $\vec{X}\in\mathcal{X}\subset\mathbb{R}^{\ell+1}$ be a random variable with $\ell+1$ dimensions $\vec{X}_0$ through $\vec{X}_\ell$. That is, each sample draw $\vec{x}^i$  has $\ell+1$ real-valued components $\vec{x}^i_0$ through $\vec{x}^i_\ell$ where the dimensions are possibly correlated. Dimension $\vec{X}_0$ is binary and represents the status of the sensitive feature. 
For example, when the sensitive feature is gender, 1 represents female and 0 represents male.
In this paper, we consider the binary classification problem, where we assume that there is a binary label $y^i$ for each data point $\vec{x}^i$, i.e., the set of possible labels $y^i$ is denoted by $\mathcal{Y} \in \{0,1 \}$. 
We are given $n$ training samples $z^1, \cdots , z^n$, where $z^i = (\vec{x}^i,y^i) \in \mathcal{X} \times \mathcal{Y}$.

Mathematically, a classifier is a function $h:\mathbb{R}^{\ell+1}\rightarrow [0,1]$ from a set of hypothesis (possible classifiers) $\mathcal{H}$, where each input sample $\vec{x} \in \mathbb{R}^{\ell+1}$ is mapped to a value in the interval $[0,1]$; a data point $\vec{x}$ is classified to $1$ if  $h(\vec{x}) > \nicefrac{1}{2}$, and to $0$ otherwise.
The ultimate goal of a classification task is to optimize some loss function $\mathcal{L}(y, h(\vec{x}))$ over all possible functions $h \in \mathcal{H}$, when applied to the training set.
We denote by $\optimal$ the classifier that minimizes this loss function.\footnote{We do not make any assumption regarding how the class $\cH$ and/or loss function $\cL$ should be chosen. Our approach guarantees that given a class and loss function, we can always design an EL-fair classifier.}
In other words, $\optimal$ is the most accurate classifier from the set $\mathcal{H}$ of functions, where all information--including sensitive feature $\vec{x}_0$--is used to achieve the highest accuracy.\footnote{Through the whole paper, we define $\optimal$ to be the classifier from class $\cH$ that minimizes the empirical loss. In many practical settings, we can find $\optimal$ in polynomial time.}
Next, we turn to our fairness definition, articulating first the explicit dimension, then the latent one.

\begin{definition}[Explicit Discrimination]\label{def:explicit disc}
Classifier $h$ exhibits no explicit discrimination if for every pair $(\vec{x}^1,\vec{x}^2)\in\mathcal{X}^2$ such that $(\vec{x}^1_1,...,\vec{x}^1_\ell)=(\vec{x}^2_1,...,\vec{x}^2_\ell)$, regardless of  $\vec{x}^1_0$ and $\vec{x}^2_0$ (i.e., the status of the sensitive features) we have
$h(\vec{x}^1)=h(\vec{x}^2)$.
\end{definition}

\cref{def:explicit disc} captures the simple and conventional way of thinking about explicit discrimination: a fair classifier should treat two similar individuals (irrespective of their sensitive features) similarly. 
Latent discrimination is, however, less trivial to formally capture. Thus, we  diagnose latent discrimination based on a subtle indirect implication that it has. We first give the formal definition and then discuss the diagnostic intuition behind it.

\begin{definition}[Latent Discrimination] \label{def:latent disc}
Classifier $h$ exhibits no latent discrimination if for every pair $(\vec{x}^1,\vec{x}^2)\in\mathcal{X}^2$ such that $\vec{x}^1_0=\vec{x}^2_0$ (i.e., pairs with similar sensitive features) we have$:$
\begin{align}
\optimal(\vec{x}^1)=\optimal(\vec{x}^2) & \Rightarrow h(\vec{x}^1)=h(\vec{x}^2), \ and \ \\
\optimal(\vec{x}^1)>\optimal(\vec{x}^2) & \Rightarrow h(\vec{x}^1)\geq h(\vec{x}^2).
\label{eq: second latent fairness}
\end{align}

\end{definition}
In words, \cref{def:latent disc} says that flipping the order of the classes of two individuals of the \emph{same group} compared to $\optimal$ is a sign of latent discrimination. 
To see the intuition behind this definition, consider $\bar{h}$, representing the most accurate classifier that satisfies \cref{def:explicit disc} (i.e., it minimizes the loss function subject only to explicit non-discrimination). 
Here, by minimizing the loss function, we  would ideally like to get as close as possible to $h^*$, but that is not generally possible given the constraint that the information about $\vec{x}_0$ may not be used. 
Thus, the minimizer would potentially treat the other $\ell$ features differently than $\optimal$ does in order to proxy for the missing $\vec{x}_0$ attribute. 
This proxying, however, inevitably changes how the classifier treats individuals within the same group, possibly by flipping the orders between some pairs. 
This is exactly what we call latent discrimination  that we would like to control for. \cref{def:latent disc} formalizes this idea in a very operational manner.
 Indeed, in  \cref{def:latent disc} we argue that the optimal unconstrained classifier provides a non-discriminatory ordering between individuals within each group. In other words, if $h^*(\vec{x}^1) > h^*(\vec{x}^2)$ for $\vec{x}^1_0 = \vec{x}^2_0$, we can conclude $\vec{x}^1$ is more qualified than $\vec{x}^2$. 
If a classifier $h$ changes this ordering, then it could be a sign of latent discrimination.
 We are now equipped with the following definition for fairness.
\begin{definition}[EL-fair]\label{def:fairness}
Classifier $h$  is ``EL-fair'' if it exhibits neither explicit nor latent discrimination as described in \cref{def:explicit disc,def:latent disc}.
\end{definition}

Note that $\optimal$ might not be EL-fair because it could suffer from explicit discrimination as it uses all features.
\section{Characterization of Optimal Fair Classifier}\label{sec:characterization}
With a formal definition of fairness in hand, we turn to the next natural step: 
\begin{quote}
What are the characteristics of an optimal classifier that satisfies EL-fairness condition?
\end{quote}
While there is not a trivial answer to this question,
in this section we show, however, that our notion of fairness lends itself into a practical algorithmic framework with the following properties. 
First, the computation of the optimal fair classifier is  straightforward. In fact, it is not more complicated than computing the optimal unconstrained classifier $\optimal.$ 
Second, it provides an intuitive interpretation in line with the idea of \emph{controlling for} different factors traditionally used in fields such as statistics and econometrics.
Our first theoretical result  establishes the existence of an EL-fair classifier. Then, in \cref{thm:MainResult}, we characterize the optimal classifier under fairness constraints of \cref{def:fairness}. 
Finally, in \cref{thm:existence,thm:computation}, we outline the properties of a simple algorithm that computes the optimal EL-fair classifier.

\newcommand{\trivial}{An EL-fair classifier exists if the set $\mathcal{H}$ (set of all possible functions in our model) includes at least one constant function.}
\begin{theorem}\label{thm:trivial}
\trivial
\end{theorem}

\begin{proof}
	The proof of this theorem is straightforward.
	It is clear that a constant function, e.g., $h(\vec{x}) = c$,  satisfies both notions of fairness: 
	\begin{itemize}
		\item for $\vec{x}^1$ and $\vec{x}^2$ such that $(\vec{x}^1_1,...,\vec{x}^1_\ell)=(\vec{x}^2_1,...,\vec{x}^2_\ell) \Rightarrow  h(\vec{x}^1) = h(\vec{x}^2) = c$.
		\item for every pair $(\vec{x}^1,\vec{x}^2)\in\mathcal{X}$ such that $\vec{x}^1_0=\vec{x}^2_0$, we have:
		\begin{itemize}
			\item $\optimal(\vec{x}^1)=\optimal(\vec{x}^2)  \Rightarrow h(\vec{x}^1)= c = h(\vec{x}^2)$.
			\item $\optimal(\vec{x}^1)>\optimal(\vec{x}^2)  \Rightarrow h(\vec{x}^1) = c \geq c = h(\vec{x}^2)$.
		\end{itemize}
	\end{itemize}
\end{proof}

Note that (almost) all practical models used in machine learning (e.g., logistic, linear, neural net, etc) allow for constant functions, therefore, they include an EL-fair classifier. We next turn to the characterization of the optimal fair classifier. But before that, we need to give a definition that (i) is necessary for the statement of the theorem; and (ii) as we argue in \cref{sec:discussion}, is conceptually crucial to the understanding of individual fairness.

\begin{definition}[A separable classifier]\label{def: separability}
Classifier $h$ is ``separable in the sensitive feature'' if there are continuous functions $g: \mathbb{R}^2\rightarrow\mathbb{R}$ and $K:\mathbb{R}^{\ell}\rightarrow\mathbb{R}$ such that:
$ \forall \vec{x} \in\mathcal{X}$ we have $h(\vec{x})=g\left(\vec{x}_0,K(\vec{x}_1,\cdots,\vec{x}_\ell)\right).$
\end{definition}
A wide range of classifiers satisfy this intuitive definition. For instance, any logistic model can be represented by choosing an appropriate linear function for $K$ and choosing $g(z_1,z_2)\equiv \frac{e^{z_1+z_2}}{1+e^{z_1+z_2}}$. Later in the paper, we discuss the close ties between the notions of separability and individual fairness. For now, we state our main result.

\newcommand{\MainResult}{Suppose the unconstrained optimal classifier $\optimal$ satisfies the definition of separability with a given $g$. Denote by $\optimalfair$ the optimal classifier (in terms of accuracy) subject to EL-fairness criteria as described in \cref{def:fairness}. There is a $\tau^* \in \mathbb{R}$ such that for all $ \vec{x} =(\vec{x}_0,\vec{x}_1,...,\vec{x}_\ell)\in\mathcal{X}:$
\begin{itemize}
	\item if $\optimalfair(\vec{x})  > \frac{1}{2}$ then  $h^*(0,\vec{x}_1,\cdots,\vec{x}_\ell) + \tau^* > \frac{1}{2}$.
	\item if $\optimalfair(\vec{x}) < \frac{1}{2}$ then  $h^*(0,\vec{x}_1,\cdots,\vec{x}_\ell) + \tau^* < \frac{1}{2}$.
\end{itemize}}

\begin{theorem}\label{thm:MainResult}
\MainResult
\end{theorem}

\begin{proof}
	Let's denote $h^{*}_{\tau} \triangleq h^{*}(0,\vec{x}_1,...,\vec{x}_\ell) + \tau$.
	In order to prove this theorem,  we first show that for every fixed $\tau$ there is a non-decreasing function $\lambda_{\tau}  : \mathbb{R} \rightarrow \mathbb{R}$ such that $\optimalfair(\vec{x}) = \lambda_{\tau} (h^{*}_{\tau} (\vec{x}))$. 
	
	Consider the mapping $\lambda_{\tau}$ such that for each $z = h^{*}_{\tau} (\vec{x})$ we have $\lambda_\tau(z) = \optimalfair(\vec{x})$.
	First we show that for all $\vec{x}^1$ and $\vec{x}^2$ such that $h^{*}_{\tau}(\vec{x}^1) = h^{*}_{\tau}(\vec{x}^2) $, we have  $\optimalfair(\vec{x}^1) = \optimalfair(\vec{x}^2) $.  This is true because
	\begin{align*}
	h^{*}_{\tau}(\vec{x}^1) = h^{*}_{\tau}(\vec{x}^2) & \overset{(a)}{\rightarrow}  h^*(0,\vec{x}^1_1,\cdots,\vec{x}^1_\ell) = h^*(0,\vec{x}^2_1,\cdots,\vec{x}^2_\ell) \\ 
	& \overset{(b)}{\rightarrow}	\optimalfair(\vec{x}^1) = \optimalfair(\vec{x}^2),
	\end{align*}
	where (i) $ (a)$ is the result of the definition of $h^{*}_{\tau}$, and (ii) from \cref{def:latent disc}, we we conclude $(b)$ since $\optimalfair$ does not use the value of the sensitive feature.
	This guarantees that the defined mapping $\lambda_{\tau}$ is a function. 
	
	In the next step, we should prove that the function $\lambda_{\tau}$ is non-decreasing. This is true because for two $\vec{x}^1$ and $\vec{x}^2$ such that $h^{*}_{\tau}(\vec{x}^1) >  h^{*}_{\tau}(\vec{x}^2)$ we have  $h^*(0,\vec{x}^1_1,\cdots,\vec{x}^1_\ell) >  h^*(0,\vec{x}^2_1,\cdots,\vec{x}^2_\ell)$, and therefore $h^{*}(\vec{x}^1) \geq h^{*}(\vec{x}^2)$.
	Now consider two sets $\mathcal{X}_0$ and $\mathcal{X}_1$ defined as follows:
	\begin{itemize}
		\item $\mathcal{X}_1 = \{x \in \mathcal{X} : \optimalfair(\vec{x}) > \nicefrac{1}{2} \}$.
		\item $\mathcal{X}_0 = \{x \in \mathcal{X} : \optimalfair(\vec{x}) <\nicefrac{1}{2} \}$.
	\end{itemize}
	From the monotonicity of $\lambda_\tau$ (i.e., it is a non-decreasing function) we know that
	\[\forall \vec{x}^0 \in \mathcal{X}_0 \text{ and } \forall \vec{x}^1 \in \mathcal{X}_1  \rightarrow  h^{*}_{\tau}(\vec{x}^0) <  h^{*}_{\tau}(\vec{x}^1). \]
	Now define $\tau_0 = \sup_{x \in \mathcal{X}_0} h^{*}_{\tau}(\vec{x})  $, and $\tau_1= \inf_{x \in \mathcal{X}_1} h^{*}_{\tau}(\vec{x}) $. From the above, we conclude that $\tau_0 \leq \tau_1$.
	Any $\tau^{*} \in [\nicefrac{1}{2} + \tau - \tau_1, \nicefrac{1}{2} + \tau - \tau_0]$ satisfies the conditions of theorem.
\end{proof}

\cref{thm:MainResult} demonstrates that for a properly chosen $\tau^*$,\footnote{Note that the use of threshold $\nicefrac{1}{2}$ is only for the purpose of exposition. All our theoretical results will hold if the threshold is chosen adaptively.} there is an $h^{*}(0,\vec{x}_1,...,\vec{x}_\ell) + \tau^*$  that \emph{mimics} the optimal fair classifier $\optimalfair$ by recommending all the decisions that $\optimalfair$ would recommend. 
In \cref{thm:existence} we prove that, under a mild assumption, such an $h^{*}(0,\vec{x}_1,...,\vec{x}_\ell) + \tau^*$ classifier  is also EL-fair.

\newcommand{\existence}{If the function $h^*\in \mathcal{H}$ is separable in the sensitive feature $\vec{x}_0$, i.e., there is a function $g: \mathbb{R}^2\rightarrow\mathbb{R}$ such that $h^*(\vec{x}) = g(\vec{x}_0,K(\vec{x}_1,...,\vec{x}_\ell))$, and the function $g$
	 is strictly monotone in its second argument, then all classifiers of the form $h^*_{\tau} \triangleq h^*(0,\vec{x}_1,\cdots,\vec{x}_\ell) + \tau$ are fair.}

\begin{theorem}\label{thm:existence}
	\existence
\end{theorem}

\begin{proof}
	To show that $h^{*}_{\tau} \triangleq h^*(0,\vec{x}_1,\cdots,\vec{x}_\ell) + \tau$ is fair, we should consider the two following cases:
	\begin{itemize}
		\item $h^{*}_{\tau} $ does not exhibit explicit discrimination: from the definition of $h^{*}_{\tau}$ it is clear that it does not depend on the value of sensitive feature and for all pairs $(\vec{x}^1,\vec{x}^2)\in\mathcal{X}$ such that $(\vec{x}^1_1,...,\vec{x}^1_\ell)=(\vec{x}^2_1,...,\vec{x}^2_\ell)$, we have $h^{*}_{\tau}(\vec{x}^1) = h^{*}_{\tau}(\vec{x}^1)$.
		\item $h^{*}_{\tau} $ does not exhibit implicit discrimination, i.e., for $\vec{x}^1,\vec{x}^2$ such that $\vec{x}^1_0 = \vec{x}^2_0$: 
		\begin{itemize}
			\item If $h^*(\vec{x}^1)=h^*(\vec{x}^2) \Rightarrow h^{*}_{\tau}(\vec{x}^1)=h^{*}_{\tau}(\vec{x}^2)$: since $\vec{x}^1_0 = \vec{x}^2_0$ and function $g$ is strictly monotone in its second argument, we should have $K(\vec{x}^1_1,...,\vec{x}^1_\ell) =K(\vec{x}^2_1,...,\vec{x}^2_\ell)$. Therefore, we have $h^{*}_{\tau}(\vec{x}^1)=h^{*}_{\tau}(\vec{x}^2)$.
			\item If $h^*(\vec{x}^1)> h^*(\vec{x}^2) \Rightarrow h^{*}_{\tau}(\vec{x}^1)\geq h^{*}_{\tau}(\vec{x}^2)$: since $g$ is strictly monotone in its second argument, without loss of generality, we assume it is strictly increasing in its second argument. Since we have $h^*(\vec{x}^1)> h^*(\vec{x}^2)$, then we should have $K(\vec{x}^1_1,...,\vec{x}^1_\ell) > K(\vec{x}^2_1,...,\vec{x}^2_\ell)$. 
			Therefore, because $g$ is strictly increasing in its second argument, we have $g(0,K(\vec{x}^1_1,...,\vec{x}^1_\ell)) > g(0,K(\vec{x}^2_1,...,\vec{x}^2_\ell))$.
		\end{itemize}
	\end{itemize}
\end{proof}

Thus, the only further step to find the optimal EL-fair classifier, in addition to computing  $h^*$, is to search for $\tau^*$. \cref{thm:computation} shows that when the function $g$  is monotone, then searching for $\tau^*$ is quite straightforward.

\newcommand{\computation}{Assume that the function $h^*$ is separable, i.e., there is a function $g: \mathbb{R}^2\rightarrow\mathbb{R}$ such that $h^*(\vec{x})=g(\vec{x}_0,K(\vec{x}_1,...,\vec{x}_\ell))$ and the function $g$
	is strictly monotone in its second argument. Furthermore, assume $\tau^*$ is the value of $\tau$ such that it maximizes the classification accuracy of $h^*_{\tau} \triangleq h^*(0,\vec{x}_1,\cdots,\vec{x}_\ell) + \tau$. The function $h^*_{\tau^*}$ is the optimal EL-fair classifier.}

\begin{theorem}\label{thm:computation}
\computation
\end{theorem}

\begin{proof}
	From \cref{thm:existence} we know that all the classifiers in the form of $h^*_{\tau}$ are fair.
	Now denote by $T$ the set of $\tau$ values that maximizes the accuracy of $h^*(0,\vec{x}_1,\cdots,\vec{x}_\ell) + \tau$. Suppose that, on the contrary to the statement of the theorem, there is a $\bar{\tau}\in T$ such that accuracy of $h^*(0,\vec{x}_1,\cdots,\vec{x}_\ell) + \bar{\tau}$  is less than $\optimalfair$.
	Note that  (i) $h^*(0,\vec{x}_1,\cdots,\vec{x}_\ell) + \tau$ is fair for all values of $\tau$ and  (ii) 
	from \cref{thm:MainResult} we know there is at least one $\tau$ such that the classification accuracy of $h^*_{\tau}$
	is the same as accuracy of $\optimalfair$. This is in contradiction with the definition of set $T$.
\end{proof}

The above property makes the search for an optimal EL-fair classifier practical.  
That is, no matter how large the dataset is, as long as $h^*$ can be computed, $\optimalfair$ can be too. 
We call this approach the \emph{train-then-mask} algorithm for eliminating latent discrimination.  Algorithm~\ref{alg:train-then-mask} describes train-then-mask.

\begin{algorithm}
	\caption{The Train-Then-Mask Algorithm} \label{alg:train-then-mask}
	\begin{algorithmic}[1]
		\State Compute the optimal classifier $h^*(\vec{x}_0,\vec{x}_1,\cdots,\vec{x}_\ell) $ over all available features $\vec{x}_0,\vec{x}_1,\cdots,$ and $\vec{x}_\ell$.
		\State Keep the sensitive feature $\vec{x}_0$ fixed (e.g., define $\vec{x}_0 = 0$) for all data points.
		\State Find the value of $\tau^*$ such that it maximizes the accuracy of $h^*(0,\vec{x}_1,\cdots,\vec{x}_\ell) + \tau^* $ over the validation set.
	\end{algorithmic}
\end{algorithm}

In spite of the fact that our formal definition of fairness is indirect, that is it turns  to \emph{within-group} variation to capture a concept that is essentially only meaningful between groups, \cref{thm:MainResult} provides an intuitive characterization.
Basically, to prevent other variables from proxying  a sensitive feature, we must \emph{control for} the sensitive feature when estimating the parameters that capture the importance of other nonsensitive variables. Crucially, the sensitive feature should not be left out of the model \emph{before} training. 
In contrast, we do not want the sensitive feature to impact our prediction/evaluation when all else is equal (to ensure individual or explicit fairness). 
This is why the sensitive feature does eventually need to be excluded \emph{after} training. 
\cref{thm:MainResult,thm:existence,thm:computation} connect the  less intuitive \cref{def:fairness} to this simple and established algorithmic procedure.

\textbf{Generalization to a set of sensitive features:} In many applications, there might be more than one sensitive feature (e.g., both gender and race might be present). It is straightforward to generalize our framework for such cases.
All of the definitions, theorems, algorithms, and interpretations remain intact if instead of $\vec{x}_0\in\mathbb{R}$ we assume $\vec{x}_0\in\mathbb{R}^m$ for some $m\in\mathbb{N}$, where $m$ is the number of sensitive features. 
Thus, our framework accommodates multiple sensitive features.
More specifically, to apply our method we first train the model on all features. In the prediction step, we keep all the sensitive attributes fixed for all data points (e.g., if the sensitive features are age and gender we assume all people are young and female). The value of $\tau^*$ is then chosen in the way to maximize the accuracy on the validation set.

\section{Discussion\label{sec:discussion}}
In this section, we further discuss several important features of our proposed fairness notion and the algorithmic solution. 
In particular, (i) we overview the relationship with the important concept of group fairness, and (ii) we further elaborate on the significance of the separability property. We also argue that separability is a central notion in understanding the individual fairness property.

\subsection{Relationship with Group Fairness} \label{sub:Group Fairness}
Unlike other suggested solutions to the problem of proxying, our approach does not incorporate some notion of group fairness to alleviate this issue. For example
\citet{kamiran2009classifying} suggested  \emph{massaging} the training set in order to exhibit group fairness, or \citet{zemel2013learning} directly incorporated group fairness into the loss function. 
Although in \cref{sec:experiments} we show that our model performs well on the group fairness measure, it has not been directly incorporated into the objectives of our model. 
The reason we avoid mixing group fairness with the problem of proxying (which is essentially a matter of individual fairness) is the potential for what we call \emph{double unfairness}, a concept which we discuss below.

Double unfairness can happen when the protected group performs better than the unprotected group in spite of the discrimination. 
For instance, consider a dataset on college admissions with two groups $\mathcal{A}$ (the protected group) and $\mathcal{B}$ (the unprotected group): (i) A person from group $\mathcal{A}$, on average, has a lower chance of admission to the college compared to a person from group $\mathcal{B}$ with the same SAT score and extracurricular activities; (ii) Nevertheless, group $\mathcal{A}$ does better than group $\mathcal{B}$ on the SAT by a wide enough margin that on average the admission rate for $\mathcal{A}$ is higher than that for $\mathcal{B}$. 
The following synthesized dataset (see \cref{tab:toyDataset}) illustrates an example for this potential scenario.

\begin{table}[ht]
\caption{Toy example: for the sensitive attribute, $1$ represents the protected group $\mathcal{A}$ and $0$  represents group $\mathcal{B}$.}
\label{tab:toyDataset}
  	\vspace{-8pt}
\centering
{ \small
\begin{tabular}{lrrrr}
\toprule
ID & Admission & Sensitive & SAT & Extracurricular \\ 
\midrule
1 &   1 &   1 & 1600 &   4 \\ 
  2 &   1 &   1 & 1500 &   6 \\ 
  3 &   1 &   1 & 1500 &   4 \\ 
  4 &   0 &   1 & 1400 &   6 \\ 
  5 &   1 &   0 & 1400 &   6 \\ 
  6 &   1 &   0 & 1300 &   5 \\ 
  7 &   0 &   0 & 1200 &   4 \\ 
  8 &   0 &   0 & 1200 &   4 \\ 
  \bottomrule
\end{tabular}
}
\end{table}

It can be seen, from \cref{tab:toyDataset}, that the admission process has been unfair to applicants from the protected group $\mathcal{A}$. 
Candidates 4 and 5 are identical with the sole exception that candidate 4 is from group $\mathcal{A}$ and 5 is from group $\mathcal{B}$. 
Candidate 4 has been denied but 5 has been admitted.\footnote{A more precise way to detect unfairness against the protected group would be to run a model (such as linear regression) on the data and observe that the coefficient on the sensitive feature is negative. 
This means that, on aggregate, candidates from the protected group are treated worse than similar candidates from the other group.} 
On the other hand, group $\mathcal{A}$ performs better than group $\mathcal{B}$ since they have an acceptance rate of $\nicefrac{3}{4}$ while that of group $\mathcal{B}$ is only $\nicefrac{1}{2}$. 
Thus, group $\mathcal{A}$ does on average 50\% better. 
Intuitively, if we are to alleviate the discrimination against the protected group, we should expect a classifier that gives even a higher edge than 50\% to them.

One can verify the danger of proxying in this toy example when the sensitive feature is omitted, by giving a higher (positive) weight to extracurricular activities compared to SAT score. This would happen because SAT has a higher correlation with the sensitive feature. 
Our approach does not allow for such weight adjustments since, by \cref{thm:MainResult}, it controls for the sensitive feature when training the rest of the weights. 
In doing so, train-then-mask gives a higher edge than the original 50\% to group $\mathcal{A}$ in terms of admission ratio. 
This provides an advantage over approaches that tackle the problem of proxying by forcing a notion of group fairness. 
For instance, the methodology by \citet{kamiran2009classifying} would first massage the data by relabeling applicant 3 to denied (or applicant 7 to admitted) and then train the classifier. 
The algorithm would do this to equate the admission rates between the two groups.
Thus, this algorithm tries to get the acceptance rates of group $\mathcal{A}$ closer to that of $\mathcal{B}$. 
This, clearly, will only further discriminate candidates from the protected group.
Similar concerns exist about other approaches that somehow employ a notion of group fairness to address proxying.

\subsection{Separability and Individual Fairness}
At the heart of the sufficient conditions for \cref{thm:MainResult} is the separability of function $f$ between the sensitive features and all other features. 
What separability roughly says is that using a separable classifier, one can rank two individuals of the same group without knowing what their (common)  sensitive feature is. In this section, we argue that our notion of fairness introduces separability as a central concept in the understanding of individual fairness; and sheds light on important future research directions. 

Note that the separability of $h^*$ is a sufficient (and not a necessary) condition for the existence of an optimal EL-fair classifier.  
In \cref{thm:necessity of separability}, we show that under a slightly stronger notion of fairness,  if there is a fair classifier then the separability property is also necessary. 

\begin{definition}[Strictly EL-fair]\label{def:strong fairness}
Classifier $h$ satisfies \textit{strong} fairness criteria if it satisfies \cref{def:explicit disc,def:latent disc} but instead of \cref{eq: second latent fairness} in \cref{def:latent disc}, it satisfies$:$
$h^*(\vec{x}^1)>h^*(\vec{x}^2) \Rightarrow h(\vec{x}^1)> h(\vec{x}^2).$
\end{definition}

\newcommand{\necessity}{Suppose there is a classifier that satisfies the strictly EL-fairness notion. Then, the function $h^*$ is separable in the sense of definition \ref{def: separability}, and the corresponding function $g: \mathbb{R}^2\rightarrow\mathbb{R}$ of the separable representation is strictly monotone in its second argument.}

\begin{theorem}\label{thm:necessity of separability}
\necessity
\end{theorem}
\begin{proof}
	Denote the strongly fair function by $\bar{h}$. 
	Given that it does not exhibit explicit discrimination, $\bar{h}(\vec{x})$ does not depend on $\vec{x}_0$. Therefore, there is a function $K: \mathbb{R}^{\ell}\rightarrow\mathbb{R}$ such that:
	\[\forall \vec{x} \in \mathbb{R}^{\ell + 1}:\quad \bar{f}(\vec{x})=\lambda(\vec{x}_1,...,\vec{x}_\ell)\]
	Now note that the strong fairness conditions are invertible. That is, for each pair $\vec{x}^1,\vec{x}^2\in\mathbb{R}^{\ell + 1}$ that have the same sensitive group status (i.e., $\vec{x}^1_0$ and $\vec{x}^2_0$ are both equal to some $\tilde{\vec{x}}_0$), we have:
	\begin{equation*}
	\bar{h}(\vec{x}^1)=\bar{h}(\vec{x}^2) \Rightarrow h^*(\vec{x}^1)=h^*(\vec{x}^2),
	\end{equation*}
	and
	\begin{equation*}
	\bar{h}(\vec{x}^1)>\bar{h}(\vec{x}^2) \Rightarrow h^*(\vec{x}^1)>h^*(\vec{x}^2).
	\end{equation*}
	Which implies:
	\begin{equation*}
	K(\vec{x}^1_1,...,\vec{x}^1_\ell)= K(\vec{x}^2_1,...,\vec{x}^2_\ell) \Rightarrow h^*(\vec{x}^1)=h^*(\vec{x}^2),
	\end{equation*}
	and
	\begin{equation*}
	K(\vec{x}^1_1,...,\vec{x}^1_\ell)= K(\vec{x}^2_1,...,\vec{x}^2_\ell) \Rightarrow h^*(\vec{x}^1)>h^*(\vec{x}^2).
	\end{equation*}
	This implies that for any fixed $\tilde{\vec{x}}_0$, there is a monotone function $g_{\tilde{\vec{x}}_0}:\mathbb{R}\rightarrow\mathbb{R}$ such that
	\[\forall \vec{x} \in\mathbb{R}^{\ell + 1} \ s.t. \ \vec{x}_0=\tilde{\vec{x}}_0:\quad h^*(\vec{x})=g_{\tilde{\vec{x}}_0}\left(  K(\vec{x}_1,...,\vec{x}_\ell) \right).\]
	But this can simply be rewritten as:
	\[\exists g:\mathbb{R}^2\rightarrow\mathbb{R} \ s.t. \ \forall x\in\mathbb{R}^{\ell + 1}:\, h^*(\vec{x})=g\left( \vec{x}_0,K(\vec{x}_1,...,\vec{x}_\ell) \right).\]
\end{proof}

To see the intuition, consider an $h^*$ that does not satisfy the separability and monotonicity properties: suppose the impact of a nonsensitive feature $\vec{x}_i$ on the outcome of $h^*$depends on $\vec{x}_0$, e.g., for the protected group (i.e., when $\vec{x}_0 = 1$) larger values of $\vec{x}_i$ results in a higher chance of positive classification (and vice versa).  
This means that the ordering implied   under $\vec{x}_0=0$ is different from the ordering under $\vec{x}_1=1$. As a result, there is no classifier that satisfies the required corresponding orderings  withing both groups. 

The concept of separability provides a lens through which we can systematically think about some of the recent papers on fairness. For instance, \citet{DIKL2018} propose a de-coupling technique which, although focused mainly on group fairness, is motivated precisely by the fact that the weight of a factor on the outcome might have different signs for different groups.
It is important to note that we do not claim one should only use separable models (even if not appropriate in the context) to ensure EL-fairness. 
Indeed, we argue that under the non-separability assumption: (i) our method for detecting latent discrimination does not work, and (ii) by using currently existing methods several other problems arise (explained in other works such as \cite{DIKL2018}). 

We close this discussion by mentioning a few open questions. The first is to consider a novel methodology for measuring the \emph{degree of non-separability} for general classifiers. Another important question is to detect latent discrimination in non-separable environments and to design algorithms to ensure EL-fairness in these cases.
Finally, we need a measure to identify the extent of proxying and a strategy that efficiently trades off accuracy with fairness. 
\section{Experiments}\label{sec:experiments}
In this section, we compare the performance of the train-then-mask algorithm to a number of baselines on real-world scenarios. 
In our experiments, we compare train-then-mask  
(i) to the unconstrained optimum classifier (i.e., the one that tries to maximize the accuracy without any fairness constraints), 
(ii) to a model in which only the sensitive feature has been removed from training procedure (note that this algorithm might suffer from the latent discrimination), 
(iii) to the trivial majority classifier which always predict the most frequent label, 
(vi) to a data massaging algorithm introduced by \citet{kamiran2009classifying},
and (v) to the algorithm for maximizing a utility function subject to the fairness constraint introduced by \citet{zemel2013learning}. 
In our experiments we consider linear SVMs (separable) \citep{scholkopf2002learning} and neural networks (non-separable) for the family of classifiers  $\mathcal{H}$. 
To find the value of $\tau^*$ for our optimal fair classifier, we use a validation set;
we take the value of $\tau^*$ such that it maximizes the accuracy over validation set and then we report the result of classification over the test set.

\noindent \textbf{Datasets:}
We use the \emph{Adult Income} and \emph{German Credit} datasets from UCI Repository \citep{asuncion2007uci, blake1998uci}, and COMPAS Recidivism Risk dataset \citep{compas2018data}. 
Adult Income dataset contains information about 13 different features of 48,842 individuals and the labels identifying whether the income of those individuals is over 50K a year. 
The German Credit dataset consists of 1,000 people described by a set of 20 attributes labeled as good or bad credit risks.
The COMPAS dataset contains personal information (e.g., race, gender, age, and criminal history) of 3,537 African-American and 2,378 Caucasian individuals.
The goal of the classification tasks in these datasets is to predict, respectively, the income status, credit risks and  whether a convicted individual commit a crime again in the following two years.

\noindent \textbf{Measures:}
We use the following measures to evaluate the performance of algorithms.
\emph{Accuracy} measures the quality of prediction of a classifier over the test set. It is defined by 
$\Acc  = 1 - \frac{\sum_{i=1}^{n} |y^i - \hat{y}^i|}{n},$
where $n$ is the number of samples in the test set, $y^i$ and $\hat{y}^i$ are the real and predicted labels of a test sample $\vec{x}^i$.
\emph{Admittance} measures the ratio of samples assigned to the positive class in each group. It is defined by
$\textrm{Admit}_1 = \frac{\sum_{i: \vec{x}^i_0 = 1} \hat{y}^i}{\sum_{i: \vec{x}^i_0 = 1} 1}$  and $\textrm{Admit}_0 = \frac{\sum_{i: \vec{x}^i_0 = 0} \hat{y}^i}{\sum_{i: \vec{x}^i_0 = 0} 1}.$
\emph{Group discrimination} measures the difference between the proportion of positive classifications within each one of the protected and unprotected groups, i.e.,
$\Dis = |\textrm{Admit}_1 - \textrm{Admit}_0|.$
\emph{Latent discrimination} is defined as the ratio of pairs that violates \cref{def:latent disc} to the total number of pairs in each group. More precisely, we have
\[ \LDis = \frac{\left|\optimal(\vec{x}^i) >  \optimal(\vec{x}^j), h(\vec{x}^i) < h(\vec{x}^j) |i\neq j , \vec{x}_0^i = \vec{x}_0^j\right|}{\binom{\sum_{i: \vec{x}^i_0 = 0} 1}{2} + \binom{\sum_{i: \vec{x}^i_0 = 1} 1}{2}}. \]
\emph{Consistency} measures a (rough) notion of individual fairness by assuming the prediction for data samples that are close to each other should be (almost) similar.
More precisely, it provides a quantitative way to compare the classification prediction of a model for a given sample $\vec{x}^{i}$ to the set of its $k$-nearest neighbors (denoted by $kNN(\vec{x}^{i})$), i.e.,
$\Knn(\vec{x}^i)  = \frac{1}{k} \sum_{j \in kNN(\vec{x}^{i})} \hat{y}^j.$

We should mention that the admittance ratios in all these datasets are always lower for the protected groups. 
For example, in the Adult Income dataset, while the income status of 31\% of the male population is positive, this value is 11\% for females. In addition, in all these datasets there are several attributes that are highly correlated with the sensitive feature. 
For example, in German Credit dataset, the correlation of the sensitive feature, i.e., ``age'', with ``Present employment since", ``Housing'' and 
``Telephone'' features are 0.24, 0.28 and 0.21, respectively.

We first consider the linear SVM classifiers which are separable.
As shown in Table~\ref{tab:results}, train-then-mask represents the best performance in terms of removing the latent discrimination (see $\LDis$).
Indeed, both discrimination measures are lower under train-then-mask than it is under the unconstrained model or the model in which the sensitive feature has been omitted. This demonstrates that train-then-mask indeed helps with the issue of proxying.
More precisely, we observe that \emph{omitting the sensitive feature} has lower accuracy but also lower discrimination compared to the unconstrained classifier.

We also observe that train-then-mask performs very well in reducing the group discrimination at the expense of a very little decrease in the accuracy.  To see this, let us compare train-then-mask to the data massaging technique \cite{kamiran2009classifying}. Under the Adult Income dataset, train-then-mask achieves higher accuracy than data-massaging but yields also higher $\Dis$. Under the German Credit dataset, it does better on both the accuracy and group discrimination fronts. These results, combined with the intuitive interpretation of our algorithm, as well as its straightforward computation, suggests that train-then-mask as an algorithm  can be easily employed to alleviate (explicit and latent) discrimination in various datasets.
This observation further demonstrates that although \cref{def:fairness} does not seem directly related to discrimination between groups, it does capture a symptom of latent discrimination. 
 
 We should point out that in our applications (and a lot of practical ones) the sensitive feature does indeed increase the accuracy of the model.
 Note that, for example, in the Adult Income dataset the accuracy of admitting all individuals, i.e., the trivial baseline classifier, is $0.756$; thus going from $0.826$ to $0.824$ is not ``negligible''.
 Our main claim is not that we do not lose much accuracy. We argue that train-then-mask, compared to other approaches that aim at resolving proxying, does well. It sometimes offers both a higher accuracy and a lower discrimination than other approaches (i.e., it dominates them) and it is never dominated in our experiments by any other approach.
 
To investigate the effect of our algorithm on discrimination for non-separable classifiers  we consider a neural network with three hidden layers. 
In \cref{tab:results}, we observe that our algorithm performs well in reducing the discrimination while maintaining the accuracy for neural network classifiers.\footnote{Note that the algorithm of \citet{zemel2013learning} does not depend on the choice of the family of classifiers $\mathcal{H}$.}
It is important to point out that in neural networks because the classifiers are not separable and it is possible to have higher levels of proxying, the latent discrimination (i.e., $\LDis$) is also increased in comparison to the SVM classifier.
 Note that  even though our theory holds only for separable classifiers, we find that our notion of fairness is relevant in other practical scenarios where classifiers are not separable.

\begin{table}
	\caption{Comparison between the performance of different algorithms on Adult, German and COMPAS datasets. Data massage algorithm refers to the method presented by  \cite{kamiran2009classifying}. The results of algorithms with the best performance on $\Acc,\Dis$  and $\LDis$ are represented in blue. These results are the average of ten experiments.
		The unconstrained model $\optimal$, as we expect, in almost all cases is the most accurate classifier. 
		Train-then-mask shows the best performance in terms of removing the latent discrimination (i.e., $\LDis$), while its accuracy is also close to the unconstrained model and even better in one case. 
		In addition, while the main goal of Train-then-mask is to remove the effect of the latent discrimination, it also performs better than the other algorithms in terms of reducing the group discrimination.
		Note that for the algorithm of  \cite{zemel2013learning} and Majority the discrimination metrics do not provide any meaningful information, because their outputs for all instances are always constant. We left out the $\LDis$ for the optimal unconstrained classifier $\optimal$ since this metric measures the latent discrimination with respect to $\optimal$ itself.}
	\label{tab:results}
	\centering
	{ \footnotesize
		\begin{tabular}{lrrrrr}
			\toprule
			&  \multicolumn{5}{c}{Adult Income dataset}      \\
			\cmidrule(r){2-6}
			Algorithm     & $\Acc$ &  $\textrm{Adm}_1$ & $\textrm{Adm}_0$ & $\Dis$  &$\LDis$ \\
			\midrule
			Unconstrained $\optimal$ (SVM) &  \textbf{\textcolor{blue}{0.825}} & 0.078  & 0.248
			& 0.170 &  - \\
			Omit sensitive feature (SVM) &  	 0.824 & 0.080 & 0.243 & 0.163 & 0.016 \\
			\textbf{Train-then-mask (SVM)} & 0.823 & 0.096 & 0.188 &  0.092  & \textbf{\textcolor{blue}{0.000}} \\
			Data massage (SVM) &  0.807 & 0.183 & 0.236 & \textbf{\textcolor{blue}{0.053}} &  0.109 \\
			\citet{zemel2013learning} & 0.756 &  0.000 & 0.000 &  0.000  & 0.000  \\
			Majority & 0.756 &  0.000 & 0.000 &  0.000  & 0.000  \\
			\midrule
			Unconstrained $\optimal$  (NN) & \textbf{\textcolor{blue}{0.825}} & 0.093 & 0.266 & 0.173 & -  \\
			Omit sensitive feature  (NN) & 0.824  & 0.092 & 0.259 & 0.167 &0.083 \\
			\textbf{Train-then-mask (NN)} & 0.823 & 0.091  & 0.192 & 0.101 & \textbf{\textcolor{blue}{0.058}} \\
			Data massage (NN) & 0.808 & 0.183 & 0.247 & \textbf{\textcolor{blue}{0.064}} & 0.146   \\
			\bottomrule
		\end{tabular}
		\begin{tabular}{lrrrrr}
			\toprule
			&   \multicolumn{5}{c}{German Credit dataset}     \\
			\cmidrule(r){2-6} 
			Algorithm     & $\Acc$ &  $\textrm{Adm}_1$ & $\textrm{Adm}_0$ & $\Dis$  &$\LDis$ \\
			\midrule
			Unconstrained $\optimal$  (SVM)  &
			\textbf{\textcolor{blue}{0.75}}& 0.60&  0.86 & 0.26 &  - \\
			Omit sensitive feature  (SVM) 
				& 0.73 & 0.64 & 0.87 &  0.23 & 0.016  \\
			\textbf{Train-then-mask (SVM)}
			& 0.74 & 0.61 &  0.80 & \textbf{\textcolor{blue}{0.19}} & \textbf{\textcolor{blue}{0.000}} \\
			Data massage (SVM)
		& 0.73  & 0.63 & 0.83 &  0.20 & 0.114 \\
			\citet{zemel2013learning}  &
			0.67 & 1.00 & 1.00 & 0.00 & 0.000  \\
			Majority &
			0.67 & 1.00 & 1.00 & 0.00 & 0.000 \\
			\midrule
			Unconstrained $\optimal$  (NN) &
				0.72 & 0.62 & 0.85 & 0.23 & - \\
			Omit sensitive feature  (NN) 
			& 0.69 & 0.67 & 0.84 & 0.17 & 0.325 \\
			\textbf{Train-then-mask (NN)} 
		& \textbf{\textcolor{blue}{0.73}} & 0.60 & 0.76 & \textbf{\textcolor{blue}{0.16}} &
		\textbf{\textcolor{blue}{0.264}} \\
			Data massage (NN) &
		0.70 & 0.62 & 0.81 & 0.19 & 0.391  \\
			\bottomrule
		\end{tabular}
		\begin{tabular}{lrrrrr}
			\toprule
			&   \multicolumn{5}{c}{COMPAS Recidivism dataset}        \\
			\cmidrule(r){2-6}
			Algorithm     & $\Acc$ &  $\textrm{Adm}_1$ & $\textrm{Adm}_0$ & $\Dis$  &$\LDis$ \\
			\midrule
			Unconstrained $\optimal$  (SVM) 
			& \textbf{\textcolor{blue}{0.768}} & 0.27 & 0.62 & 0.35  &  - \\
			Omit sensitive feature  (SVM) 
			&    0.765 & 0.34 & 0.56 & 0.22 &  0.005  \\
			\textbf{Train-then-mask (SVM)}
			& 0.766 & 0.44 & 0.64  & \textbf{\textcolor{blue}{0.20}} & \textbf{\textcolor{blue}{0.000}} \\
			Data massage (SVM)
				& 0.747 & 0.35 & 0.58 & 0.23 & 0.025 \\
			\citet{zemel2013learning}  &
			0.509 & 	1.00 & 1.00 & 0.00 & 0.000  \\
			Majority & 	 0.509 & 1.00 & 1.00 & 0.00 & 0.000  \\
			\midrule
			Unconstrained $\optimal$ (NN)
		& \textbf{\textcolor{blue}{0.767}} &0.25 &0.60 & 0.35 & - \\
			Omit sensitive feature  (NN) 
				& 0.741  & 0.42 &  0.67 & 0.25 &  0.095 \\
			\textbf{Train-then-mask (NN)} 
			& 0.740 & 0.31 & 0.53 & \textbf{\textcolor{blue}{0.22}} & \textbf{\textcolor{blue}{0.076}}\\
			Data massage (NN) &
		 0.738 & 0.39 & 0.63 & 0.24 & 0.098 \\
			\bottomrule
		\end{tabular}
	}
\end{table}

 \noindent \textbf{Effect of Threshold $\bm{\tau}$:}
 In \cref{thm:existence}, we showed that under certain conditions for all value of $\tau$, the function $h^*(0,\vec{x}_1,\cdots,\vec{x}_\ell) + \tau$ is fair based on \cref{def:fairness}. 
 \cref{thm:computation} argues that the value of $\tau^*$  that maximizes the accuracy has the same classification outcome as the optimal fair classifier. 
 However, if in an application one is more interested in lowering discrimination than in accuracy, she may choose values of $\tau$ other than $\tau^*$ in order to fine-tune the accuracy-discrimination trade-off according to the specifics of the application. 
 \cref{fig:tau} illustrates this point on the Adult Income, German Credit, and COMPAS Recidivism Risk datasets. 
 Each sub-figure consists of points in the accuracy-discrimination space where each point comes from a specific choice of $\tau$. 
 The orange part of each curve (circles) is the \emph{Pareto frontier}. 
 That is, one cannot choose a $\tau$ that does better on both accuracy and discrimination fronts than an orange colored point. Whereas any blue points (triangles) correspond to choices of $\tau$ that are dominated by one other choice of $\tau$ on both fronts.
 The black point (square) corresponds to the value of $\tau^*$ with the maximum accuracy on the validation set.
 Note that accuracy and discrimination values reported in \cref{tab:results} are for this choice of $\tau^*$.

\noindent \textbf{Consistency and Individual Fairness:}
In this part, we provide experimental evidence to support our claim regarding the individual fairness of our proposed classifier.
For this reason, we compare the consistency in the output of our fair classifiers, i.e., for each sample $\vec{x}^i$, we compare the value of $\hat{y}^i$ to $\Knn(\vec{x}^i)$. 
In \cref{fig:knn}, we observe that, not only for two data samples $\vec{x}^1$ and $\vec{x}^2$ such that $(\vec{x}^1_1,...,\vec{x}^1_\ell)=(\vec{x}^2_1,...,\vec{x}^2_\ell)$ the prediction is exactly the same (explicit individual fairness), for data samples which are close to each other (based on their Euclidean distances in the feature space) the predictions remain close.

\begin{figure*}[htb!]
	\centering
	\subfloat[Adult Income dataset]{\includegraphics[height=1.3in]{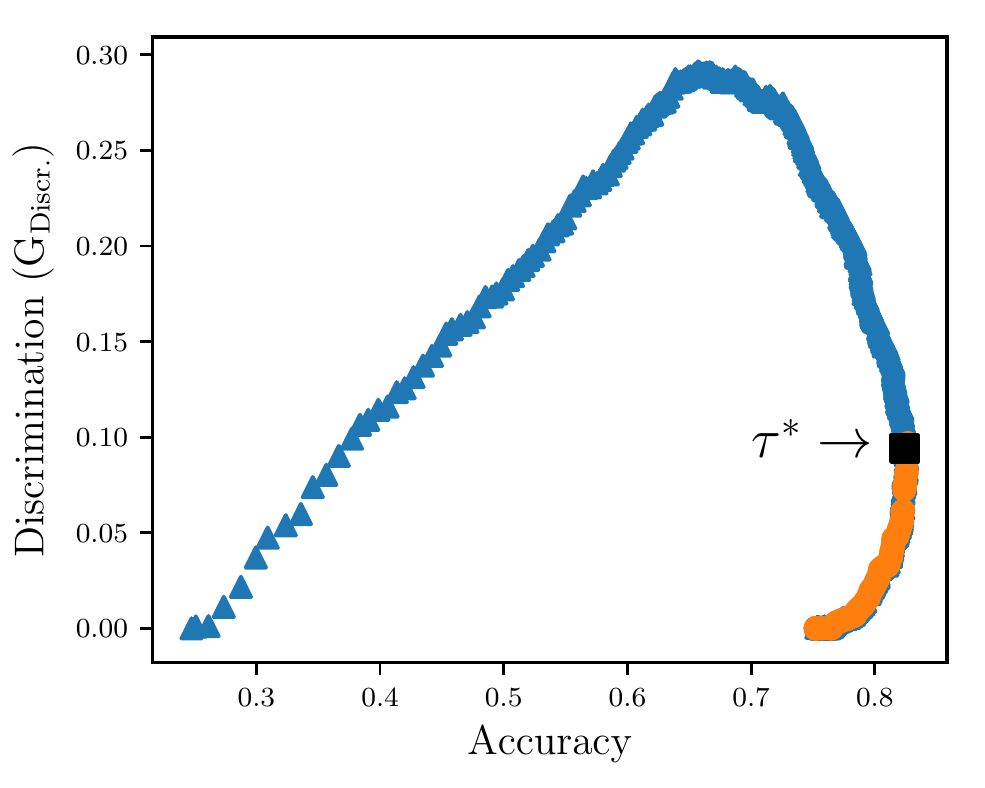} \label{fig:adult}}
	\quad
	\subfloat[German Credit dataset]{\includegraphics[height=1.3in]{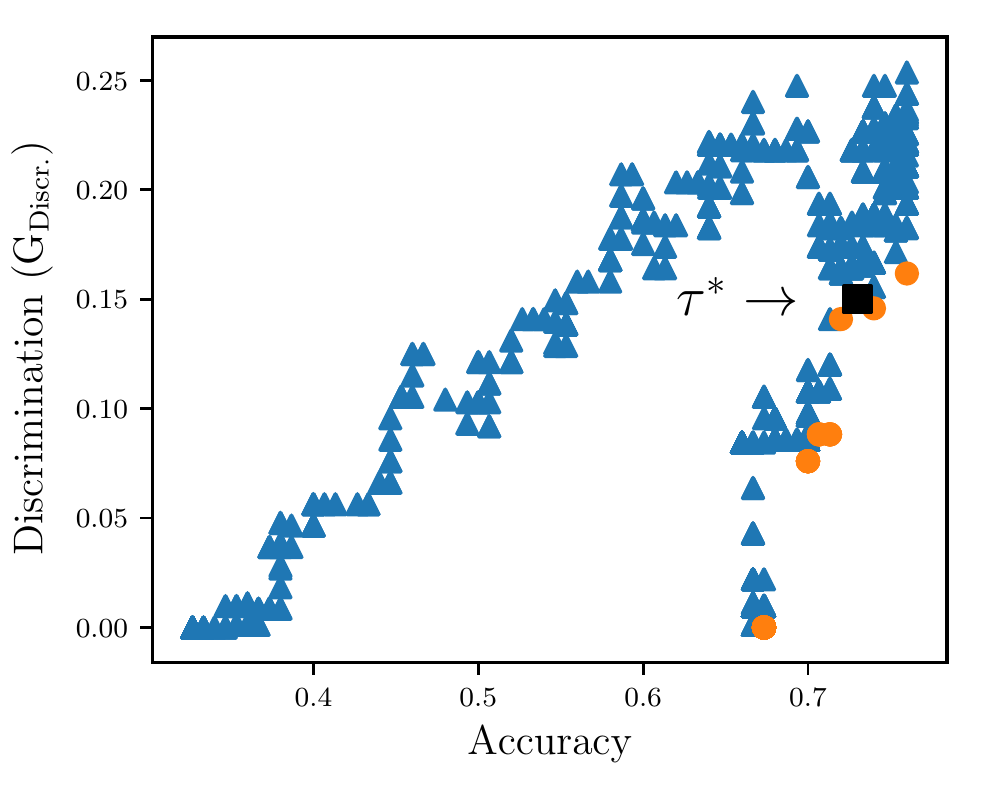}\label{fig:german}}
	\quad
	\subfloat[COMPAS Recidivism dataset]{\includegraphics[height=1.3in]{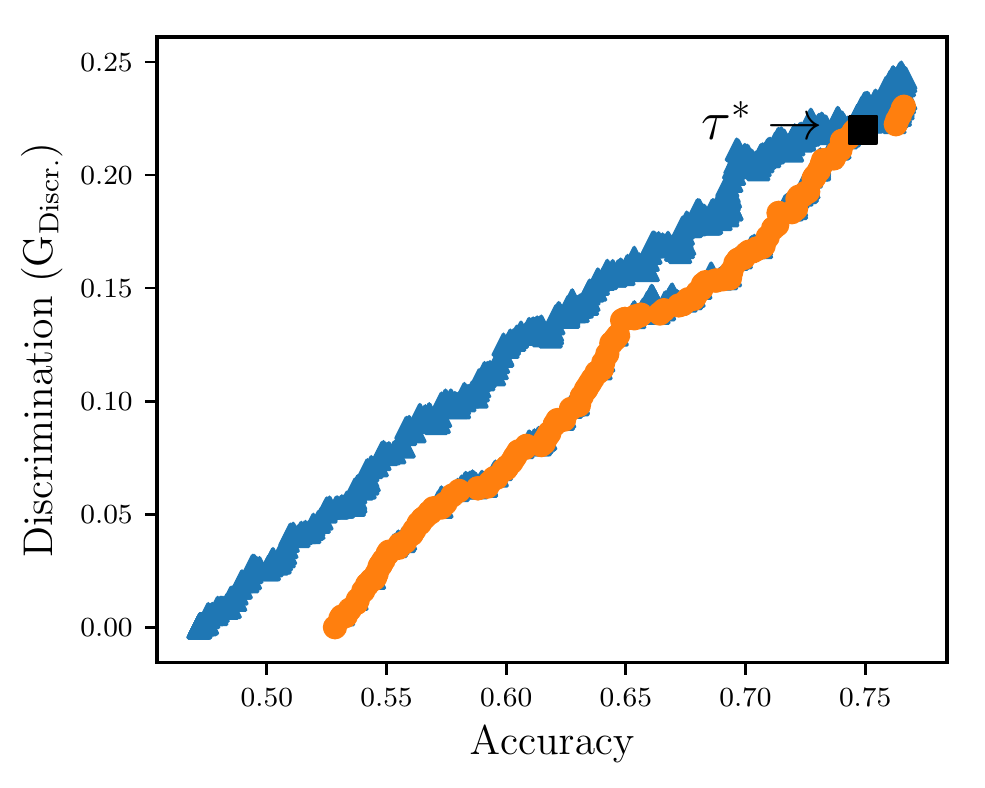}\label{fig:compasn}}
	\caption{The trade-off between accuracy and fairness for different values of $\tau$. The orange points represent the Pareto frontier samples of $\tau$, which means that one cannot do better than them in both fronts of accuracy and discrimination. The black point (square) corresponds to the value of $\tau^*$ with the maximum accuracy on the validation set.}\label{fig:tau}
\end{figure*}
\begin{figure*}[htb!]
	\centering
	\subfloat[Adult Income dataset ($k = 10$)]{\includegraphics[height=1.3in]{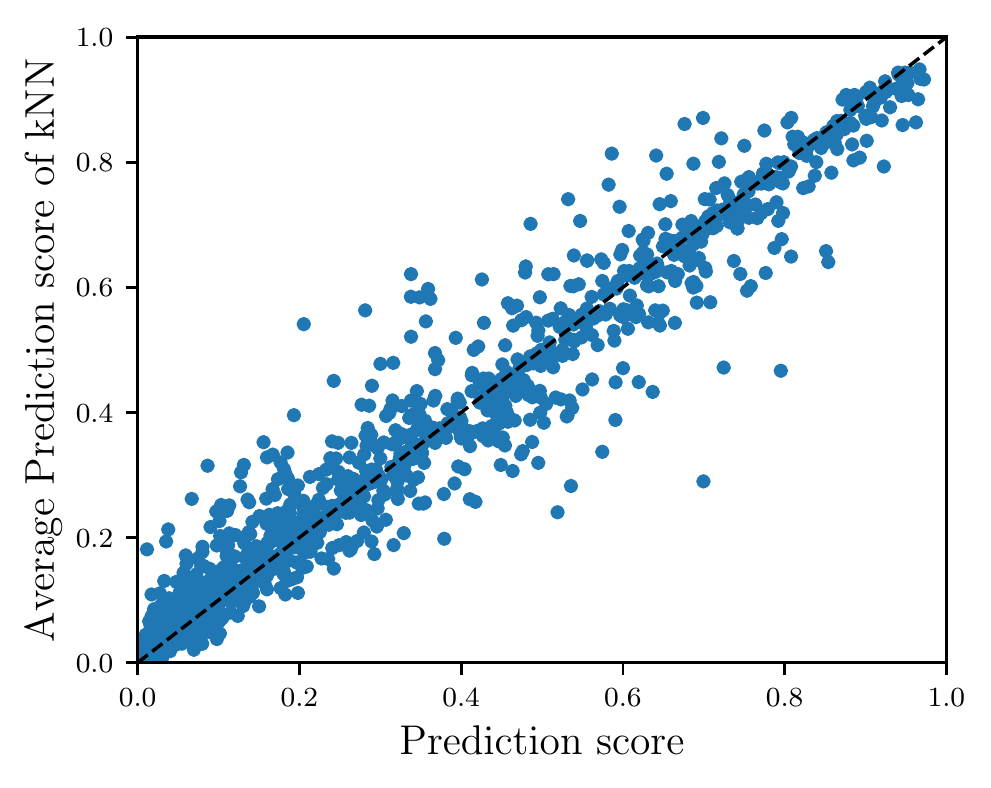} \label{fig:adult_knn}}
	\quad
	\subfloat[German dataset ($k = 5$)]{\includegraphics[height=1.3in]{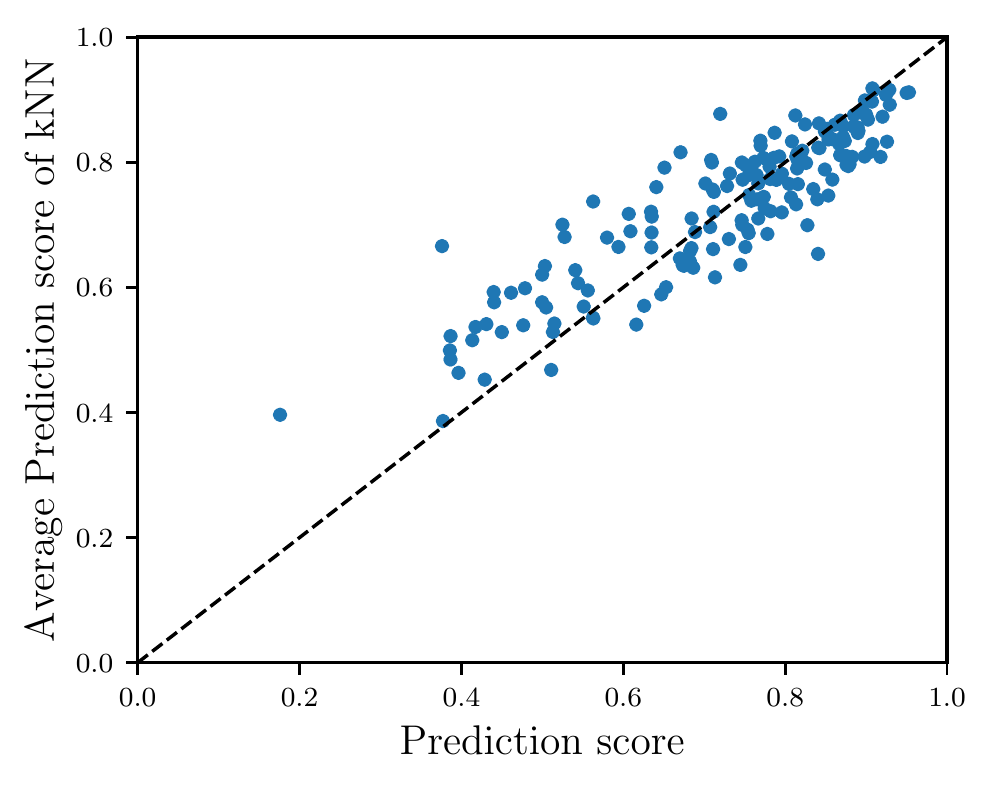}\label{fig:german_knn}}
	\quad
	\subfloat[COMPAS dataset ($k = 5$)]{\includegraphics[height=1.3in]{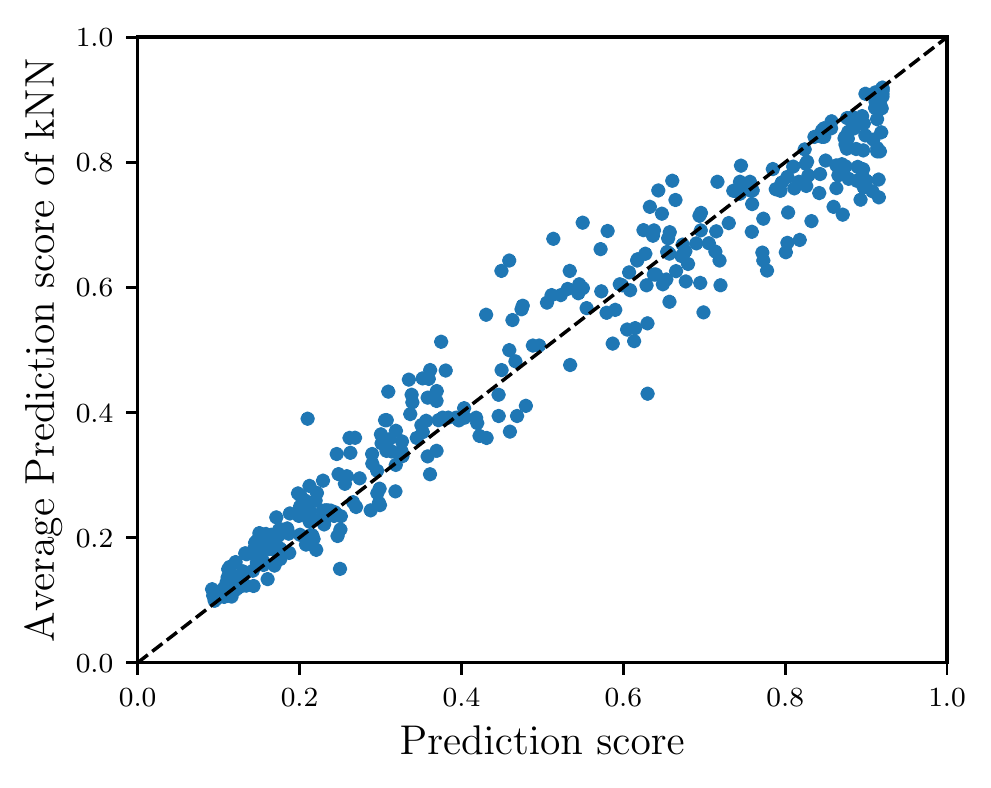}\label{fig:compas_knn}}
	\caption{The comparison between the prediction score of each data point and the average score of its $k$ nearest neighbors in the test set. We set $\tau = 0$ in this experiment. Note that any other value of $\tau$ shifts both axes equally.
	}\label{fig:knn}
\end{figure*}

\noindent \textbf{Multiple Sensitive Features:}
In \cref{sec:characterization}, we discussed that our framework generalizes to scenarios where there are multiple sensitive features.
In this part, we use an SVM classifier to evaluate the performance of our algorithm over multiple sensitive features.
In the first experiment, we consider ``sex'' and ``race'' as two sensitive features in Adult Income dataset.
In the prediction step, we assume all people are ``female" and ``black".
The  accuracy is 0.824  with  $\textrm{Adm}_1 = 0.091$, $\textrm{Adm}_0 = 0.174$  and $\Dis = 0.083$ for ``sex".
For the second experiment, in the German Credit dataset, we consider ``Personal status and sex'' along with the ``age'' as sensitive attributes. Also, in the training step, we assume all individuals are ``young'' and  "female and single''. The  accuracy is 0.71  with  $\textrm{Adm}_1 = 0.78$, $\textrm{Adm}_0 = 0.86$  and $\Dis = 0.08$ for ``age". For these experiments, when we na\"ively omit the sensitive features $\Dis$ are, respectively,  0.170 and 0.196 for Adult Income and German Credit datasets.
\section{Conclusion}\label{sec:conclusion}
It is well established in the literature that simply omitting the sensitive feature from the model will not necessarily give a fair classifier. 
This is because often, nonsensitive features are correlated with the sensitive one and can act as proxies of that feature, bringing about latent discrimination. 
In spite of the consensus on the importance of latent discrimination and the attempts to eliminate it, no formal definition of it has been provided based on the notion of within-group fairness.
Our main observation for providing an operational definition of latent discrimination relied on diagnosing this phenomenon by examining its symptoms. 
We argue that changing the order of the values assigned to two samples within the same group compared to the optimal unconstrained classifier is a symptom of proxying, and we call a classifier free of latent discrimination if it does not exhibit any such disorders.

We demonstrated that our notion of fairness has multiple favorable features, making it suitable for analysis of individual fairness. 
First, we proved that the optimal fair classifier can be represented in a simple fashion. 
It enjoys an intuitive interpretation that the sensitive feature should be  omitted after, rather than before training. 
This way, we control for the sensitive feature when estimating the weights on other features; but at the same time, we do not use the sensitive feature in the decision-making process. Based on this intuition, we then provided a simple two-step algorithm, called \emph{train-then-mask}, for computing the optimal fair classifier. 
We showed that aside from simplicity and ease of computation, our notion of fairness had the advantage that it does not lead to \emph{double discrimination}. 
That is when the group that is discriminated against is also the group that performs better overall, our method still removes the bias against that protected group. 
Finally, we should point out while $\optimal$ can be computed in many practical scenarios, but in the worst case, it is not possible to have $\optimal$. Therefore, there is a gap between the surrogated accuracy and real accuracy which results in a gap between surrogated and real discriminations.

\paragraph{Acknowledgements.} The work of Amin Karbasi was supported by AFOSR Young Investigator Award (FA9550-18-1-0160).

\bibliography{EliminatingLatentDiscrimination}
\bibliographystyle{plainnat}
\end{document}